\documentclass[letterpaper]{article} 
\usepackage{aaai24}  
\usepackage{times}  
\usepackage{helvet}  
\usepackage{courier}  
\usepackage[hyphens]{url}  
\usepackage{graphicx} 
\usepackage{natbib}  
\usepackage{caption} 
\usepackage{algorithm}
\usepackage{algorithmic}

\usepackage{newfloat}
\usepackage{listings}
\usepackage{amsfonts}
\usepackage{multirow}
\usepackage{subcaption}
\usepackage{fontawesome}
\usepackage{amsthm}
\usepackage{verbatim}
\DeclareCaptionStyle{ruled}{labelfont=normalfont,labelsep=colon,strut=off} 
\floatstyle{ruled}
\newfloat{listing}{tb}{lst}{}
\floatname{listing}{Listing}
\title{TACIT: A Target-Agnostic Feature Disentanglement Framework for Cross-Domain Text Classification}
\author{
    Rui Song,\textsuperscript{\rm 1} 
    Fausto Giunchiglia,\textsuperscript{\rm 1, \rm 2, \rm 3}
    Yingji Li,\textsuperscript{\rm 3} 
    Mingjie Tian,\textsuperscript{\rm 1}
    Hao Xu\thanks{Corresponding author}\textsuperscript{\rm 1, \rm 3, \rm 4}
}
\affiliations{
    \textsuperscript{\rm 1}School of Artificial Intelligence, Jilin University, Changchun 130012, China\\
    \textsuperscript{\rm 2}Department of Information Engineering and Computer Science, University of Trento, Via Sommarive 938123, Trento, Italy\\
    \textsuperscript{\rm 3}College of Computer Science and Technology, Jilin University, Changchun 130012, China\\
    \textsuperscript{\rm 4}Chongqing Research Institute, Jilin University, Chongqing 401123, China\\
    \{songrui20, yingji21, mjtian19\}@mails.jlu.edu.cn,
    fausto@disi.unitn.it,
    xuhao@jlu.edu.cn
}

\usepackage{bibentry}

\begin{document}

\maketitle

\begin{abstract}
Cross-domain text classification aims to transfer models from label-rich source domains to label-poor target domains, giving it a wide range of practical applications. Many approaches promote cross-domain generalization by capturing domain-invariant features. However, these methods rely on unlabeled samples provided by the target domains, which renders the model ineffective when the target domain is agnostic. Furthermore, the models are easily disturbed by shortcut learning in the source domain, which also hinders the improvement of domain generalization ability. To solve the aforementioned issues, this paper proposes TACIT, a target domain agnostic feature disentanglement framework which adaptively decouples robust and unrobust features by Variational Auto-Encoders. Additionally, to encourage the separation of unrobust features from robust features, we design a feature distillation task that compels unrobust features to approximate the output of the teacher. The teacher model is trained with a few easy samples that are easy to carry potential unknown shortcuts. Experimental results verify that our framework achieves comparable results to state-of-the-art baselines while utilizing only source domain data.
\end{abstract}

\section{Introduction}
In recent years, natural language processing (NLP) models based on deep networks have made significant progress and have even surpassed human-level performance. But these methods often rely on manually labeled data, and the inconsistency between the distribution of labeled training domains and the unlabeled target domains poses a challenge for deploying these methods in practical applications~\cite{Ben-DavidRR20}. To address this challenge, Unsupervised Domain Adaptation (UDA) has emerged as a solution. UDA aims to generalize a model trained on labeled data from source domains to perform well on a target domain without labeled data. By employing UDA, models can overcome their dependency on labeled data from the target domain, which has attracted considerable attention from researchers. 

In UDA, cross-domain text classification is a basic but challenging task because of the differences in text expressions among the source and target domains. To enhance the performance of cross-domain text classification, numerous researches have focused on improving the generalization ability by extracting domain invariant features, including pivot-based methods~\cite{ZiserR18, ZhangHPJ18, Ben-DavidRR20}, task-specific knowledge-based methods~\cite{ZhouTWWXH20}, domain adversarial training methods~\cite{WuS22}, and class-aware methods~\cite{YeTHLNB20, LuoGL022}. Besides, there are approaches that use language models to perform self-supervised tasks to capture task-agnostic features in the target domain~\cite{KarouzosPP21}. These methods take full advantage of the commonality between the source and target domains to encourage model generalization.

However, these approaches still face \textbf{two main challenges}. \textbf{First}, capturing domain-invariant features tend to depend heavily on the target domain, which makes the model ineffective when the target domain is agnostic. Besides, the training of the generalized model requires consideration of additional target domain samples, which adds training and deployment costs. \textbf{Second}, the models are susceptible to shortcut learning\footnote{Shorcuts are defined as simple decision rules that can not be applied to more challenging scenarios, such as cross-domain generalization. Shortcut learning occurs when the model relies excessively on superficial correlations in the source domain, disregarding domain-specific features crucial for accurate classification in the target domain. Therefore, mitigating shorcuts can improve cross-domain generalization~\cite{MoonMLLS21}. In our study, shortcuts are not predefined, but included in easy samples. } in the source domain, which also hinders the improvement of domain generalization ability~\cite{GeirhosJMZBBW20, Yunlong2023}. 

To overcome the above challenges, we propose a \textbf{T}arget-\textbf{A}gnostic framework for \textbf{C}ross-domain text class\textbf{I}fica\textbf{T}ion (TACIT). It is inspired by the work of feature disentanglement~\cite{HuangZCWY21} as well as Variational Auto-Encoder (VAE) for text generation~\cite{BaoZHLMVDC19}. The aim of TACIT is to separate robust and unrobust features from the potential latent feature space of the source domain, and use robust features to promote cross-domain generalization performance. Moreover, we design a feature distillation task to encourage further separation of the unrobust features from the robust features. The teacher model in the distillation task learns from easy samples in the training set to ensure that it itself carries unrobust features. As a result, TACIT can use only source domain samples for cross-domain text classification without any target domain data and additional target domain training. Experiments on four publicly available datasets confirm that TACIT is capable of going beyond state-of-the-art approaches. Overall, our contributions are as follows:
\begin{itemize}
	\item We propose a feature disentanglement framework for separating robust and unrobust features and facilitating the model's ability to generalise across domains in target domain agnostic scenario.
	\item We train an unrobust teacher model with easy samples, and design a feature distillation task to encourage further decoupling of unrobust features. 
	\item We experimentally confirm that the proposed TACIT can be compared with some of the most advanced methods in the absence of target domain data.
\end{itemize}

\section{Related Work}\label{sec:work}
In this section, we list some of the work related to TACIT, including cross-domain text classification and entanglement methods in NLP.

\subsection{Cross-Domain Text Classification}
The high cost of acquiring large amounts of labeled data for each domain has prompted research into cross-domain text classification with the help of Unsupervised Domain Adaptation techniques~\cite{BlitzerDP07, YuJ16, RamponiP20}. Most of the previous works facilitate generalization by capturing pivots common to source and target domains~\cite{LiWZY18, ZiserR18, ZhangHPJ18, Ben-DavidRR20}, where pivots are key features or attributes that act as a bridge, enabling the model to transfer knowledge learned from labeled source data to the unlabeled target data. Another common approaches are domain adversarial training, which enhance the generalization ability of the model by allowing it to distinguish between source domain and target domain data~\cite{GaninUAGLLML16, QuZC0Z19, WuS22}. Task-specific knowledge-based methods introduce additional task-related knowledge to facilitate generalization. For example, SENTIX uses existing lexicons and annotations at both token and sentence levels to retrain the language model~\cite{ZhouTWWXH20}. \cite{LiWJZ22} helps cross-domain generalization by extracting sentiment-driven semantic graphs from Abstract Meaning Representation. Class-aware methods extracts better category-invariant features by learning more discriminative source domain labels~\cite{YeTHLNB20, LuoGL022}. Besides, there are approaches that use language models to perform self-supervised tasks to capture task-agnostic features in the target domain~\cite{DuSWQL20, KarouzosPP21}. They re-train the language model by performing cloze tasks in the target domain, so that the features of the target domain can be captured without any labels.

In contrast to the previous research methods, our approach adopts a more stringent criteria where the target domain is completely agnostic, and even unlabeled texts are not provided. This means there is no need to retrain the model specifically for the target domain when performing a new task in that domain. This flexibility allows for seamless application of our approach across diverse target domains without any target-specific training requirements.

\subsection{Textual Feature Disentanglement}
The disentanglement of latent space is first explored in the field of computer vision, and features of images (such as rotation and color) have been successfully disentangled~\cite{ChenCDHSSA16}. In NLP tasks, it is used to address the decoupling of latent representations of text, such as text style and content~\cite{JohnMBV19}, syntax and semantics~\cite{BaoZHLMVDC19}, opinions and plots in user reviews~\cite{PergolaGH21}, fairness representation and bias against sensitive attributes~\cite{ColomboSNP22}. They rely on Variational Auto-Encoders or some variations~\cite{KingmaW13}, to restore the original feature from the space of disentanglement. In addition, there are methods to facilitate the separation of specific feature spaces by imposing regularization constraints on different tasks~\cite{JohnMBV19, HuangZCWY21}. In this paper, inspired by the above disentangled methods, we promote the effect of cross-domain text classification by separating robust and unrobust features.

\section{Proposed Framework}
This section elaborates on the proposed framework TACIT. First, to facilitate the narration, we first give the problem formulation and some symbolic definitions. Subsequently, we gradually describe the composition of TACIT as shown in Figure~\ref{fig:model}. TACIT mainly contains a student model based on VAE and an easy teacher model with unrobust features. In the process of feature disentanglement of the student, the separated unrobust features are encouraged to learn from the teacher for better decoupling effect.

\begin{figure*}[t]
	\centering
	\includegraphics[width=1.5\columnwidth]{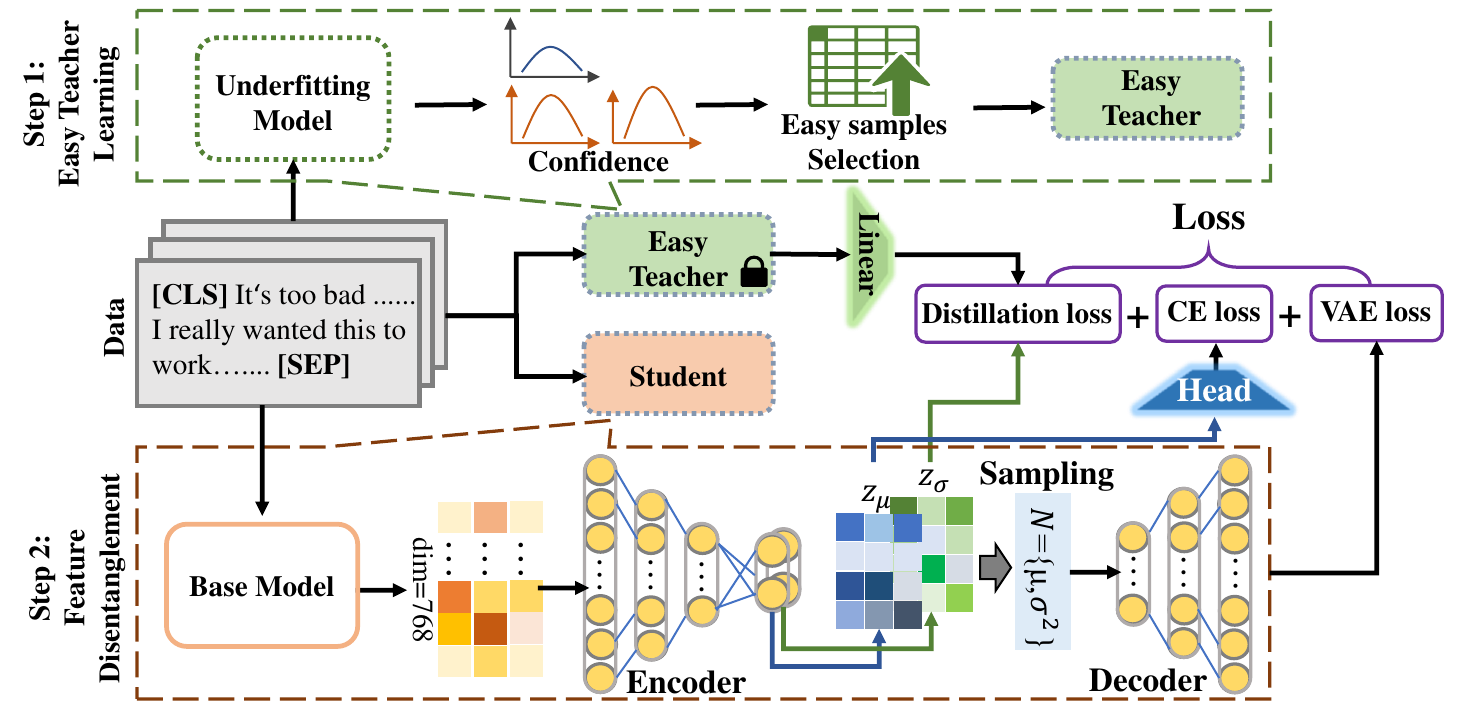} 
	\caption{TACIT's overall architecture and processing flow. It consists of two main steps and three tasks. In Step 1, an underfitting model selects a subset of easy samples from the source domain based on the confidence. Subsequently, such samples are used to train a teacher model. In Step 2, the output features of the base model are fed into VAE for disntanglement. The robust feature $z_\mu$ is used to predict the sample labels. Then, the unrobust feature $z_\sigma$ is scheduled to be learned from the teacher's output $\hat{z}$ through feature distillation. Finally, cross-entropy loss, VAE loss and distillation loss are used to co-optimize the model. \faLock ~indicates that model parameters are not updated during training.}
	\label{fig:model}
\end{figure*}

\subsection{Problem Formulation}
Similiar to~\cite{WuS22}, we consider two different scenarios: a source domain relative to a target domain and multiple source domains relative to a target domain. For any number of source domains $\mathcal{S}^l=\{x_i^l, y_i^l\}_{i=1}^{N^l_s}$ with labeled datas, our goal is to get a fully trained language model $\mathcal{M}$ with a classification head $\mathcal{F}(.)$, which has good generalization ability in the target domain $\mathcal{T}=\{x_i^t\}_{i=1}^{N_t}$ without any label. Here, $N^l_s$ and $N_t$ represents the number of samples from different source domains and the target domain, where $l\geq 1$ denotes the minimum number of source domain is 1. For any text $x_i$, it contains $m+2$ tokens $\{[CLS], t_1, ... t_m, [SEP]\}$ where $[CLS]$ is used to obtain the representation $h_i$ of the text output by $\mathcal{M}$. Then, $\mathcal{F}(h_i)$ maps the representation to the appropriate label $y_i$. In some methods, despite the absence of any labeling, data from the target domain $\mathcal{T}$ can still help train $\mathcal{M}$. In our approach, only the source domain $\mathcal{S}^l$ can be used.

\subsection{Student: Feature Disentanglement Based on VAE}
Under the premise that the target domain is agnostic, we expect that the model can disentangle robust and unrobust features in the continuous latent feature space, and the former is used for effective cross-domain generalization, while the latter is discarded as task-irrelevant features. Inspired by some related work on textual feature disentanglement~\cite{BaoZHLMVDC19, JohnMBV19}, we adopt VAE to separate robust and unrobust features from sample feature space~\cite{KingmaW13}. 

Specifically, we use a probabilistic latent variable $z$ to encode the representation $h$, and then decode $h$ from $z$:
\begin{equation}
	p(h) = \int p(z)p(h|z)~\mathrm{d}z,
\end{equation}
where $p(z)$ denotes the prior which is the standard normal $\mathcal{N}(0, \mathrm{I})$. To optimize VAE, the following loss according to the evidence lower bound(ELBO) is defined:
\begin{equation}
	\mathcal{L}_{vae} = -\mathbb{E}_{q(z|h)}[logp(h|z)] + KL(q(z|h)||p(z)),
\end{equation}
where $q(z|h)$ is the posterior given by the decoder, which is formed by $\mathcal{N}(\mu, \mathrm{diag}~\sigma^2)$. $KL$ is Kullback-Leibler divergence. Here, $\mu$ and $\sigma^2$ can be regarded as independent of each other under the premise of the standard normal~\cite{kawata1949characterisation,Fotopoulos07}, we present the relevant proof in the Appendix\ref{app:a}. Therefore, we use their corresponding representations to represent robust and unrobust features, instead of a simple feature split of $z$~\cite{JohnMBV19}. In practice, they can be modeled by two independent linear transformations and represented as $z_{\mu}$, $z_{\sigma}$. 

Next, to ensure the robustness of $z_{\mu}$, it should be able to help the model make correct predictions. Therefore, a classification head is used to predict the label of the current sample from $z_{\mu}$ by cross entropy (CE):
\begin{equation}
	\mathcal{L}_{ce} = CE(softmax(Head(z_{\mu}))),
\end{equation}
where $Head(.)$ is modeled using a linear transformation, which maps the input representations to the latent label space. By optimizing the above loss, it is possible to ensure the effectiveness of robust features for the classification task.

\subsection{Teacher: Easy Samples Learning}
Now, we have two premises, $z_{\mu}$ and $z_{\sigma}$ are disentangled and $z_{\mu}$ is used for robust label prediction. Several studies have shown that additional tasks targeting different features can help to further disentangle the features~\cite{JohnMBV19, HuangZCWY21}. Therefore, a natural idea is to add an extra task for $z_{\sigma}$ making it produce unrobust predictions. With difficulty, when the target domain is agnostic, producing unrobust predictions that are not conducive to cross-domain generalization is unavailable. Therefore, we train an easy teacher model for generating unrobust features and guiding $z_{\sigma}$ indirectly. The acquisition of teacher model is inspired by some unknown biases mitigation approaches~\cite{UtamaMG20, UtamaMG20b}, where a shallow model is easily affected by easy samples. We expect to extract easy samples from the training set and train the teacher model to learn the unrobust features contained in the easy samples. 

\textbf{Easy Samples Extraction}. Previous studies have proved that easy samples can be easily fitted by models with fewer parameters~\cite{LaiZFHZ21}. Besides, the model is also more likely to make overconfident predictions for the easy samples~\cite{DuMJDDGSH21}. Therefore, we obtain an underfitting shallow model to determine whether the sample is an easy sample. Specifically, we train a DistilBERT\footnote{https://huggingface.co/distilbert-base-uncased} or DistilRoBERTa\footnote{https://www.huggingface.co/distilroberta-base} for 2 epochs on all the training samples and rank the samples based on confidence. Confidence denotes the largest value in the predicted probability distribution. So if a sample can obtain a large confidence in the case of underfitting, it could be an easy samples. Top 35\%\footnote{In practice, the top 30\% samples are sometimes difficult to guarantee that the teacher model can give intelligent predictions, so we choose a slightly higher proportion.} of the samples are considered as easy samples. 

\textbf{Teacher training}. Subsequently, the easy samples are fed into a new distillation model for teacher learning. Unlike the underfitting models described above, we expect the teacher model to capture as much knowledge as possible from the easy samples, so the teacher model is trained until convergence. The training process for the teacher model is the same as for the student model, with details in Section~\ref{sec:details}.

\subsection{Distillation: Unrobust Features Distillation}
Different from most previous distillation methods for distilling logits, the unrobust features in TACIT do not perform the label prediction task. Therefore, our approach works on features as the distillation target. For each sample in the training set, the teacher model produces an unrobust feature $\tilde{h}$, even if the sample is not an easy sample. To align with $z_\sigma$, $\tilde{h}$ is fed to a simple linear transformation to get $\tilde{z} \in \mathbb{R}^{64}$. To make two different features comparable, we normalize them with a whitening operation, which is implemented by a non-parametric layer normalization operator without scaling and bias~\cite{Yixuan2022}. Then, a smooth $l_1$ loss is used as the loss function for feature distillation:
\begin{equation}
	\mathcal{L}_{distill} = \left\{
	\begin{array}{ll}
		\frac{1}{2}(\mathcal{\zeta}(z_\sigma)-\mathcal{\zeta}(\tilde{z}))^2/\beta, & |\mathcal{\zeta}(z_\sigma)-\mathcal{\zeta}(\tilde{z})|\leq \beta  \\
		|\mathcal{\zeta}(z_\sigma)-\mathcal{\zeta}(\tilde{z})| - \frac{1}{2}\beta, & otherwise, 
	\end{array} \right.
\end{equation}
where $\zeta(.)$ indicates the whitening operation, $\beta$ is a fixed parameter set to 2.0. By optimizing the above loss, the entangled feature $z_\sigma$ can be approached to the features of the unrobust teacher, thus further achieving separation from the robust features.

\subsection{Training and Inference}
Finally, the main body of training is the student model, so the overall loss function is the joint loss of three different loss functions:
\begin{equation}
	\mathcal{L} = (1-\lambda_1-\lambda_2)*\mathcal{L}_{ce} + \lambda_1*\mathcal{L}_{vae} + \lambda_2*\mathcal{L}_{distill},
\end{equation}
where $\lambda_1$ and $\lambda_2$ are the weighted coefficient. Throughout the training process, all parameters of the teacher are frozen as it only provides prior knowledge of unrobust features. 

In the process of inference, the encoder part of the student model is used to predict the label of a new sample by the robust feature $z_\mu$, regardless of whether the new sample comes from the source or target domain. After the above process, TACIT does not require any unlabeled data from the target domain for domain adversarial training, but only uses the source domain data to obtain robust model. 

\section{Experiments}\label{sec:exp}
In this section, we present the datasets required for the experiments, the baselines for comparisons, the results in the single-source and multi-source domains, and the corresponding experimental results with the framework analysis. 

\subsection{Datasets}\label{sec:data}

We evaluate the proposed TACIT on the most widely used Amazon reviews dataset~\cite{BlitzerDP07}, which contains binary sentiment classification tasks from four different domains: Books (B), DVDs (D), Electronics (E), and Kitchen (K). Each domain contains 1000 positive samples and 1000 negative samples. For each domain, we use a five-fold cross-validation protocol, where 20\% of the samples are randomly selected as the development set, and the optimal model on the development set is saved for the target domain generalization test. Publicly available data divisions are used to make fair comparisons~\cite{Ben-DavidRR20}. Then, compliance with the previous work~\cite{WuS22}, we give different configurations for the single-source and multiple-source cases. For single-source domains, we train on one dataset and test on the other three. Thus a total of 4*3 = 12 tasks are constructed~\cite{ZiserR18}. For multi-source domains, we train the model on any three datasets and test it on the remaining one. Thus a total of 4*1=4 tasks are constructed. In addition to the widely used Amazon reviews dataset, we have also compared the proposed approach on a variety of tasks and models. See Appendix~\ref{app:b} for the details and results.

\subsection{Baselines}\label{sec:baseline}
We compare TACIT with the following state-of-the-art approaches to validate the competitiveness of the proposed method:
\begin{itemize}
	\item \textbf{DAAT}~\cite{DuSWQL20}. It encourages BERT to capture domain-invariant features through domain-adversarial training, thus improving generalization capabilities. 
	\item \textbf{R-PERL}~\cite{Ben-DavidRR20}. It extends BERT with a pivot-based variant of the Masked Language Modeling (MLM) objective. 
	\item \textbf{CFd}~\cite{YeTHLNB20}. It introduces class-aware feature self-distillation to self-distill PLM's features into a	feature adaptation module, which makes the features from the same class are more tightly clustered. 
	\item \textbf{UDALM}~\cite{KarouzosPP21}. It continues the pretraining of BERT on unlabeled target domain data using the MLM task, then trains a task classifier with source domain labeled data. 
	\item \textbf{COBE}~\cite{LuoGL022}. It improves the contrastive learning loss of negative samples within one batch, so that the representations of different classes become further away in the potential space. It is more generalizable on similar tasks by giving more reasonable determinations on categories. 
	\item \textbf{AdSPT}~\cite{WuS22}. It trains vanilla language model with soft prompt tuning and an adversarial training object, thus alleviating the domain discrepancy
	of MLM task.
	\item \textbf{Vanilla}. For comparison, we also fine-tune the basic language models BERT~\cite{DevlinCLT19} and RoBERTa~\cite{RoBERTa} with the cross-entropy loss function. 
\end{itemize}
To evaluate the baselines, we use accuracy as the evaluation metric following~\cite{KarouzosPP21, WuS22}. With the exception of AdSPT and CFd, which use RoBERTa and XLM-R, the other approaches use BERT as the backbone language model. We report the optimal results given in the original papers to prevent duplicated code from failing to achieve the results reported in the paper. In addition, we also replicate UDALM using RoBERTa as the backbone based on the official code for a fair comparison, as it is the optimal model on BERT. 

\subsection{Experimental Details}\label{sec:details}
We initialize our model with BERT$_{base}$ and RoBERTa$_{base}$ as the backbone. Accordingly, to ensure that student models can be aligned to the teachers, the teacher models are the corresponding distillation versions, DistilBERT~\cite{DistilBERT} and DistilRoBERTa. All models are trained 10 epochs with batch size 64. The learning rate is set to 1e-5, and the optimizer is AdamW~\cite{LoshchilovH19}. The weight of the loss function is set to $\lambda_1=0.001$ and $\lambda_2=0.1$ (See Section~\ref{sec:params} for a detailed discussion). For Encoder and Decoder, we use two symmetric three-layer MLPs where the activation function is ReLU and the hidden layer sizes are 356, 128, and 64, respectively. All experiments are conducted with Pytorch and HuggingFace Transformers on four NVIDIA GeForce RTX 2080 Ti GPUs. Our code is available online\footnote{https://github.com/songruiecho/TACIT}.

\subsection{Results}\label{sec:results}

\begin{table*}[htbp]
	\centering
	\setlength\tabcolsep{1.6mm}
	\renewcommand{\arraystretch}{0.5}
	\begin{tabular}{cccccccc|cccc}
		\hline \hline 
		\multirow{2}[0]{*}{Source→Target} & \multicolumn{7}{c}{BERT} & \multicolumn{4}{c}{RoBERTa} \\ \cline{2-12}
		& Vanilla & DAAT & R-PERL & CFd   & COBE  & UDALM & TACIT & Vanilla & UDALM & AdSPT & TACIT \\ \hline 
		B→D & 88.96 & 89.70 & 87.80  & 87.65 & 90.05 & 90.97 & 91.42 & 91.45 & 92.18 & 92.00 & \textbf{92.65} \\
		B→E & 86.15 & 89.57 & 87.20  & 91.30 & 90.45 & 91.69 & 91.68 & 93.19 & 93.55 & \textbf{93.75} & 93.81 \\
		B→K & 89.05 & 90.75 & 90.20  & 92.45 & 92.90 & 93.21 & 92.73 & 93.35 & \textbf{95.32} & 93.10 & 95.03 \\
		D→B & 89.40 & 90.86 & 85.60  & 91.50 & 90.98 & 91.00 & 91.33 & 91.51 & 93.34 & 92.15 & \textbf{93.57} \\
		D→E & 86.55 & 89.30 & 89.30  & 91.55 & 90.67 & 92.30 & 91.83 & 90.42 & 93.60 & \textbf{94.00} & 93.16 \\
		D→K & 87.53 & 87.53 & 90.40  & 92.45 & 92.00 & 93.66 & 91.55 & 92.85 & 93.21 & 93.25 & \textbf{94.40} \\
		E→B & 86.50 & 88.91 & 90.20  & 88.65 & 87.90 & 90.61 & 89.62 & 91.40 & 91.80 & \textbf{92.70} & \textbf{92.70} \\
		E→D & 87.59 & 90.13 & 84.80  & 88.20 & 87.87 & 88.83 & 89.25 & 89.28 & \textbf{93.38} & 93.15 & 92.06 \\
		E→K & 91.60 & 93.18 & 91.20  & 93.60 & 93.33 & 94.43 & 94.18 & 94.95 & 94.85 & 94.75 & \textbf{95.87} \\
		K→B & 87.55 & 87.98 & 83.00  & 89.75 & 88.38 & 90.29 & 89.70 & 91.00 & 92.74 & 92.35 & \textbf{93.06} \\
		K→D & 87.95 & 88.81 & 85.60  & 87.80 & 87.43 & 89.54 & 89.20 & 89.83 & 92.33 & \textbf{92.55} & 91.97 \\
		K→E & 90.45 & 91.72 & 91.20  & 92.60 & 92.58 & 94.34 & 93.40 & 92.80 & 93.56 & 93.95 & \textbf{94.57} \\ \hline 
		Avg & 88.25 & 90.12 & 87.50  & 90.63 & 90.39 & 91.74 & 91.32 & 91.84 & 93.32 & 93.14 & \textbf{93.57} \\ \hline \hline 
	\end{tabular}%
	\caption{Single source cross-domain generalization performance for TACIT and baselines. The boldface indicates the optimal results. For each model we report the average results across the five folds. `Vanilla' denotes fine-tuning on the source domain labeled data. `Source' denotes training on the source and `Target' means testing on the target dataset. `Avg' represents the average of all cross-domain generalization tasks.}
	\label{tab:results}%
\end{table*}%

\begin{table}[htbp]
	\centering
	\setlength\tabcolsep{3.3mm}
	\renewcommand{\arraystretch}{0.5}
	\begin{tabular}{cccc}
		\hline  \hline
		Source→Target & Vanilla & AdSPT & TACIT \\ \hline
		DEK→B & 92.70 & 93.50  & \textbf{93.64} \\
		BEK→D & 91.63 & 93.50  & \textbf{95.06} \\
		BDK→E & 94.00 & 94.25 & \textbf{95.04} \\
		BDE→K & 94.15 & 93.75 & \textbf{94.86} \\  \hline
		Avg   & 93.12 & 93.75 & \textbf{94.65} \\ \hline \hline
	\end{tabular}%
	\caption{Cross-domain generalization results of multiple training sources. The boldface indicates the optimal results. \textbf{AdSPT} is the only method reporting multiple sources in the baselines, so we use RoBERTa$_{base}$ as the backbone for comparison.}
	\label{tab:multi-source}%
\end{table}%

We report the experimental results of BERT and RoBERTa as the backbones in Table~\ref{tab:results}. We find that the proposed TACIT is close to the state-of-the-art approaches in both single-source (Table~\ref{tab:results}) and multi-sources configurations (Table~\ref{tab:multi-source}), even without using any target domain samples. We also find that different data sources have different impacts on the target domain (Figure~\ref{fig:multi}). The specific descriptions and discussions are as follows.

\textbf{Results on single source}. When using BERT as the backbone, UADLM achieves the best results (91.74\%) because it performs BERT's MLM training on the target domain, which improves the ability to model the target domain context. But the average result of TACIT is only 0.42\% less than that of UDALM without additional tasks for the target domain. This proves that the proposed method can influence the performance of the target domain through reasonable modeling of the source domain. Besides, TACIT works better than CFd (0.69\%) and COBE (0.93\%), which shows that the disentanglement of robust and unrobust features is more efficient than reasonable in-class modeling.

While using RoBERTa as the backbone, we can observe that TACIT achieves the best average result, which is 0.25\% higher than UDALM. This is because RoBERTa is larger and contains more task-related knowledge than BERT. Our feature disentanglement method can assist RoBERTa to make more informed feature choices without the stimulation of the target domain. A more straightforward example can be found in Vanilla, where just fine-tuned RoBERTa yielded better result (91.84\%) than UDALM$_{bert}$, suggesting that RoBERTa is better suited to the task of cross-domain semantic classification. Therefore, in the subsequent parameter exploration and ablation experiment, we used RoBERTa as the backbone. 

\begin{figure}[t]
	\centering
	\includegraphics[width=0.9\columnwidth]{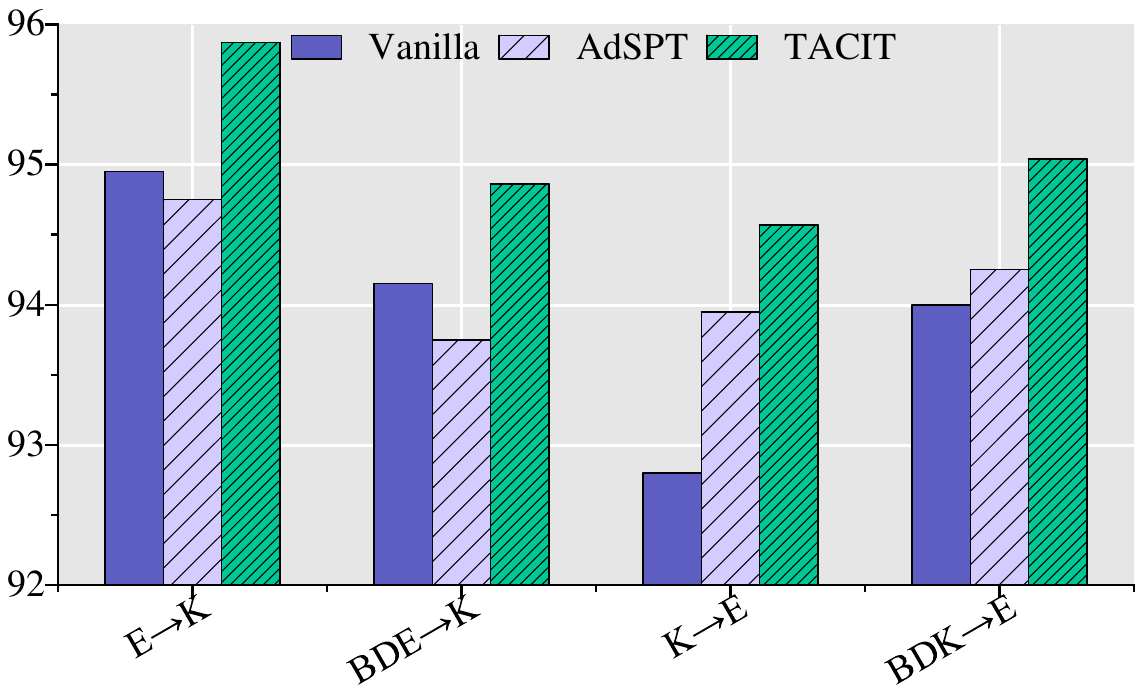} 
	\caption{Comparison of single-source and multi-source experimental results on similar data sets K and E.}
	\label{fig:multi}
\end{figure}

\textbf{Results on multi-sources}. We observe that increasing the source domain generally improves the performance of the target domain on average, due to the commonality among similar tasks. For Vanilla and TACIT, multiple source configurations create 1.28\% and 1.08\% boosts. But for AdSPT, the improvement is only 0.61\%. This suggests that AdSPT is not sensitive to changes in the source domain, which may be due to the fact that it partially relies on data from the target domain, whereas Vanilla and TACIT fully rely on the source domain. It indicates that similar multi-source configurations can better stimulate TACIT's performance improvement. But multi-source configuration is not valid in all cases, as explained in the following paragraph.

\begin{figure}[t]
	\centering
	\includegraphics[width=0.99\columnwidth]{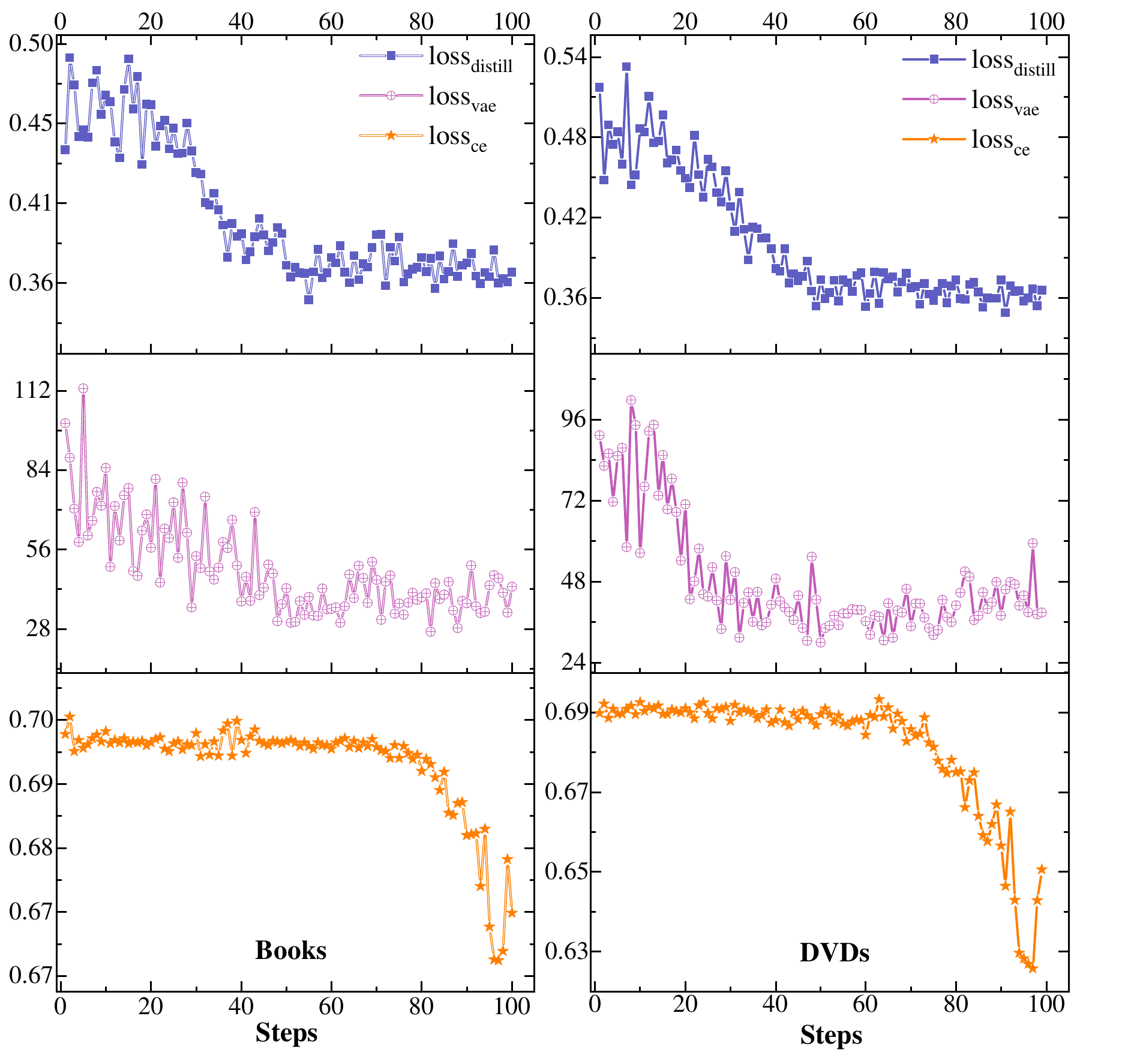} 
	\caption{The changes of loss on fold-1 with Books and DVDs as source domains during the model training process. Different styles of lines represent different datasets as well as loss values. }
	\label{fig:loss}
\end{figure}

\begin{figure*}[t]
	\centering
	\includegraphics[width=1.7\columnwidth]{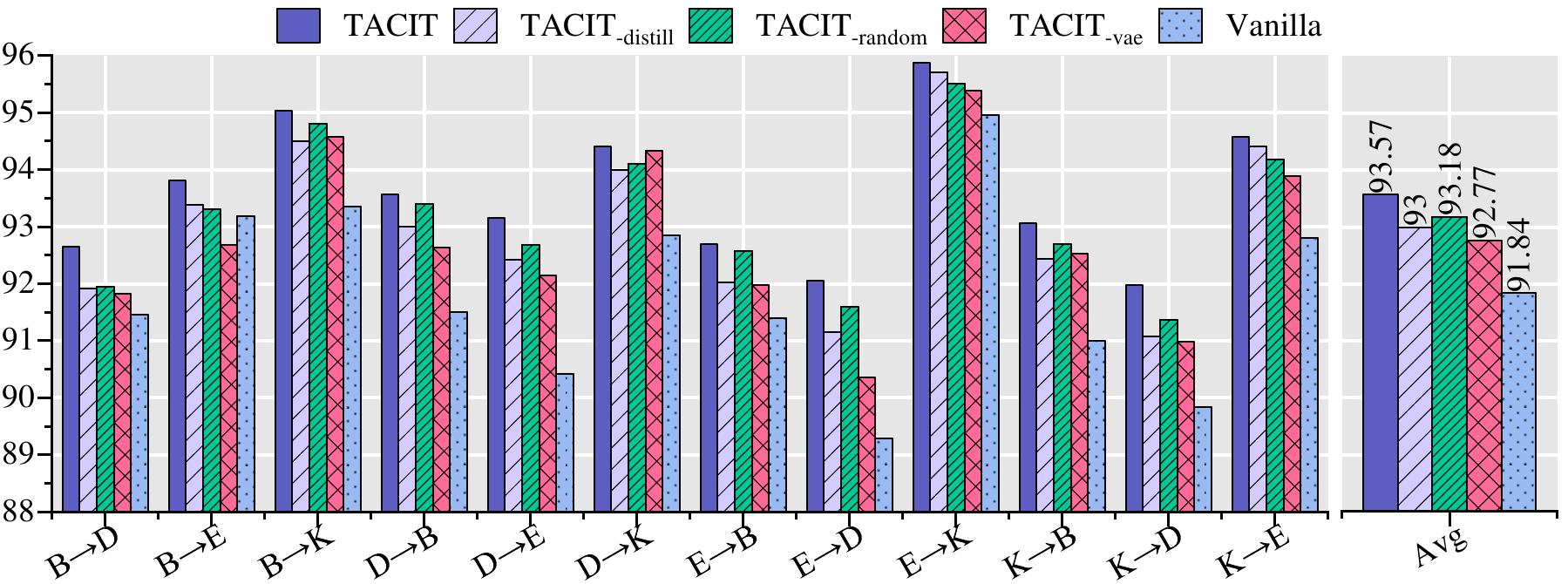} 
	\caption{Comparison of ablation results of different cross-domain generalization tasks, where different colors and styles of bars indicate different TACIT variants. }
	\label{fig:abl}
\end{figure*}

\textbf{Single-source v.s. Multi-source}. Subsequently, with the same target domain, we further compare the results for single and multiple sources in Figure~\ref{fig:multi}. We select datasets K and E with similar feature distributions. As reported in previous work, generalization between similarly distributed datasets tends to have better experimental results~\cite{WuS22}, but with the introduction of less similar datasets, a decrease in generalizability may be observed, e.g., the results of BDE$\to$K are inferior to those of E$\to$K. On the contrary, this phenomenon disappears when E is used as the target domain, which suggests that different datasets perform differently when used as source and target domains, indicating the importance of dataset selection in cross-domain text classification tasks.

\subsection{Parameter Selection}\label{sec:params}

In the cross-domain generalization task, the grid search of parameters is difficult because of the need to consider multiple target domains. Therefore, we compare the loss values for different tasks to determine a rough parameter magnitude, rather than manually adjusting for different datasets. As shown in Figure~\ref{fig:loss}, the cross entropy loss of main task and distillation loss are about the same order, while the VAE loss is much larger. So we set $\lambda_1=0.001$ and $\lambda_2=0.1$ to ensure that the auxiliary tasks do not unduly affect the optimization of the main task. Although this may lead to non-optimal results, parameter tuning in the face of a new task is economized and is more conducive to task migration easily.

\subsection{Ablation Study}\label{sec:abl}

To further verify the effectiveness of the proposed method, the following three variants of TACIT are tested:
\begin{itemize}
	\item TACIT$_{-distill}$. It means that the feature distillation module is not used, and only VAE is used for disentanglement.
	\item TACIT$_{-random}$. It means randomly selecting 35\% of the samples as the training data for the teacher model, rather than selecting the samples with high confidence. 
	\item TACIT$_{-vae}$. It means that VAE is not used, but the output of Encoder is fed directly to two different linear transformations, one whose output is used to predict labels and the other whose output is used for feature distillation. 
\end{itemize}
The results of the ablation studies are shown in Figure~\ref{fig:abl}. All three variants cause TACIT performance degradation, both in individual tasks and overall averages. But there are differences between them. Firstly, we observe that TACIT$_{-vae}$ causes the most performance degradation in all but a few cases (K$\to$B). This shows that the biggest factor affecting TACIT is sufficient decoupling of features. If VAE is removed, then the independence between features is abandoned. As a result, the robust features can not be separated. Secondly, we also observe a decline in TACIT$_{-distill}$'s performance, as it is further disentangling features through different tasks. Because it does not destroy the overall architecture of feature disentanglement, the impact is small. Thirdly, TACIT$_{-random}$ also reduces the generalization effect of the model, which shows that easy sample selection based on confidence is more advantageous than random easy sample selection. In addition, TACIT$_{-random}$ also brings minimal performance degradation, indicating that even with random sample selection, feature distillation still brings some positive effects compared to TACIT$_{-distill}$. Finally, with the exception of B$\to$E, all variants of TACIT achieve better results than Vanilla. Therefore, to sum up, the core components of TACIT all make positive contributions to the improvement of generalization.

\subsection{Visualisation}\label{sec:vis}

\begin{figure}[t]
	\centering
	\begin{subfigure}{0.46\linewidth}
		\includegraphics[width=\linewidth]{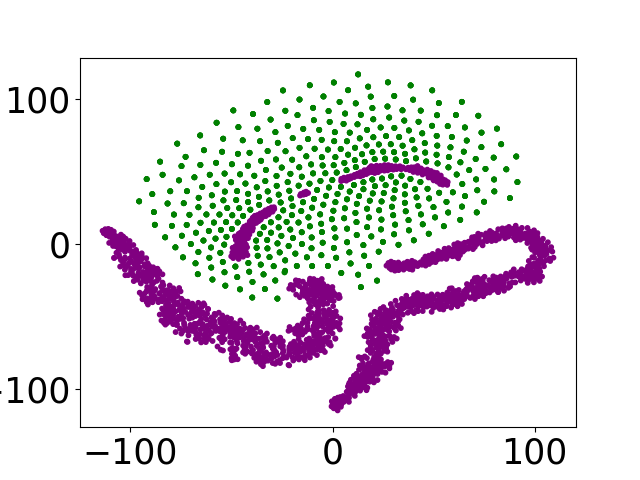}
		\caption{TACIT}
		\label{fig:sub1}
	\end{subfigure}
	\begin{subfigure}{0.46\linewidth}
		\includegraphics[width=\linewidth]{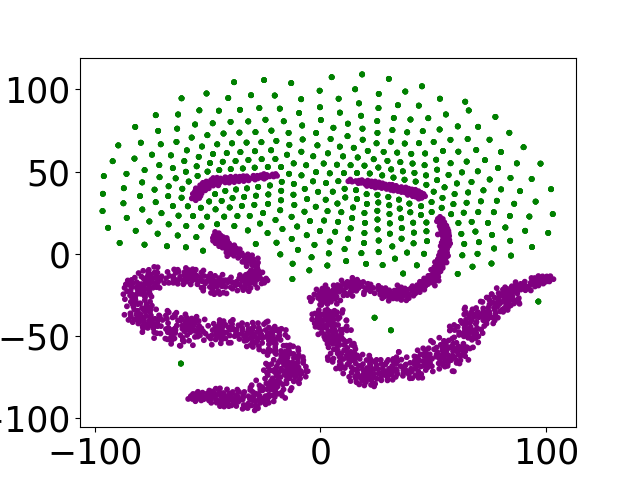}
		\caption{TACIT$_{-distill}$}
		\label{fig:sub2}
	\end{subfigure}
	\\
	\begin{subfigure}{0.46\linewidth}
		\includegraphics[width=\linewidth]{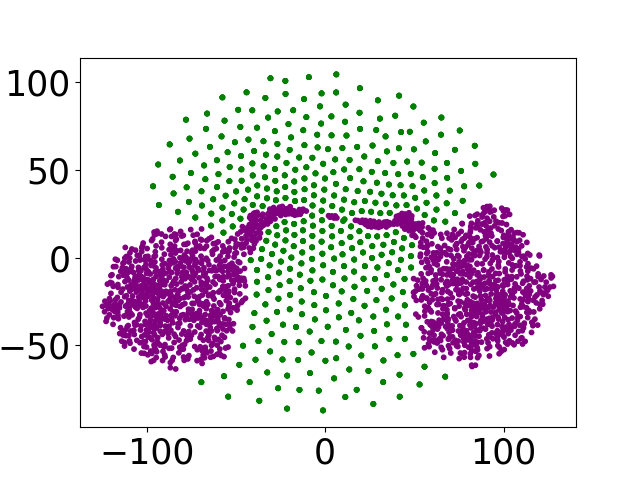}
		\caption{TACIT$_{-vae}$}
		\label{fig:sub3}
	\end{subfigure}
	\caption{Feature visualisation results of $z_\mu$ and $z_\sigma$ for TACIT and the two corresponding variants TACIT$_{-distill}$ and TACIT$_{-vae}$ on B$\to$D, where the green nodes indicate $z_\sigma$ and the purple nodes indicate $z_\mu$. }
	\label{fig:vis}
\end{figure}

Further, through the visualisation of the representations, we determine the impact of feature disentanglement on cross-domain generalisation. Specifically, tSNE is used to project 64-dimensional features into a two-dimensional space~\cite{vandermaaten08a}. In Figure~\ref{fig:vis}, we show the visualisation results of B$\to$D. The representation $z_\mu$ is used for classification so that smooth clusters can be obtained with model optimization, which can be observed in three subgraphs. But in Figure~\ref{fig:sub3}, the purple clusters are not as smooth as Figure~\ref{fig:sub1} and Figure~\ref{fig:sub2}, suggesting that VAE can enhance the results of label classification. Besides, the three subgraphs demonstrate differences in $z_\sigma$. Specifically, in Figure~\ref{fig:sub1}, green cluster is more compact and more clearly distinguishable from the purple clusters, suggesting a good separation of the two features. For TACIT$_{-distill}$ in Figure~\ref{fig:sub2}, the green cluster is much looser and some nodes demonstrate a tendency to stray, which suggests that deleting feature distillation has had some negative effects. The most significant impact on the results is the deletion of VAE as shown in Figure~\ref{fig:sub3}, which directly results in green nodes spanning the entire space. The above observations correspond to the results in Section~\ref{sec:abl}, which further illustrates the effectiveness of the proposed method for feature disentanglement.

\section{Conclusion}
In this paper, facing the challenge of target domain agnostic in cross-domain text classification, we propose a feature disentanglement framework TACIT based only on source domain. TACIT is built on the premise that robust features contribute to classification, while unrobust features are irrelevant. The disentanglement of robust and unrobust features is achieved by variational autoencoders, and this feature separation is exacerbated by additional feature distillation tasks. The experiment of common cross-domain text classification datasets proves that the proposed method can achieve comparable results as the optimal method without using any target domain data. 

In the future work, we will explore more judicious methods of easy sample selection to train a more unrobust teacher model. In addition, other language models will be further explored to rate the generalizability of the proposed method. 

\clearpage
\section{Acknowledgments}
This work was supported by National Natural Science Foundation of China (NSFC), ``From Learning Outcome to Proactive Learning: Towards a Humancentered AI Based Approach to Intervention on Learning Motivation” (No. 62077027), and the major project of the National Natural Science Foundation of China (NSFC) ``Research on Major Theoretical and Practice Issues in Innovation-Driven Entrepreneurship" (Grant No. 72091310), Project 1 ``Developing Theory on Innovation-Driven Entrepreneurship in the Digital Economy" (Grant No. 72091315). The work was also supported by the Education Department of Jilin Province, China (JJKH20200993K) and the Department of Science and Technology of Jilin Province, China (20200801002GH).

\bibliography{aaai24}

\clearpage

\appendix
\section{Proof} \label{app:a}

\newtheorem{theorem}{Theorem}
\begin{theorem}
	Suppose a set of independent samples $\{X_1, X_2, ... ,X_n\}$ follow the normal distribution $\mathcal{N}(\mu, \sigma^2)$, the mean and variance of the samples are independent of each other. 
\end{theorem}
\begin{proof}
	For the above samples, the mean and variance can be expressed as $\bar{X}=\frac{1}{n}\sum_{i=1}^{n}X_i$ and $S^2=\frac{1}{n-1}\sum_{i=1}^{n}(X_i-\bar{X})^2$. To prove that $\bar{X}$ and $S^2$ is independent of each other, an orthogonal matrix $A$ is constructed as follows:
	\begin{equation} \small  \setlength{\arraycolsep}{0.5pt} 
		A = \left[ {\begin{array}{cccccc} 
				\frac{1}{\sqrt{n}} & \frac{1}{\sqrt{n}} & \frac{1}{\sqrt{n}} & ... & \frac{1}{\sqrt{n}} & \frac{1}{\sqrt{n}} \\
				\frac{1}{\sqrt{2*1}} & \frac{-1}{\sqrt{2*1}} & 0 & ... & 0 & 0 \\
				\frac{1}{\sqrt{3*2}} & \frac{1}{\sqrt{3*2}} & \frac{-2}{\sqrt{3*2}} & ... & 0 & 0 \\
				\vdots & \vdots & \vdots & & \vdots & \vdots \\
				\frac{1}{\sqrt{n(n-1)}} & \frac{1}{\sqrt{n(n-1)}} & \frac{1}{\sqrt{n(n-1)}} & ... & \frac{-1}{\sqrt{n(n-1)}} & \frac{-(n-1)}{\sqrt{n(n-1)}} \\
		\end{array} } \right]
	\end{equation}
	Through the orthogonal matrix $A$, $X$ can be transformed into $Y$ by the orthogonal transformation $Y=AX$, where $Y=[Y_1, Y_2, ..., Y_n]^\textsuperscript{T}$. Since $Y$ can be represented by $X$, the probability density function for both can be written as: 
	\begin{equation}
		\begin{array}{cl}
			\mathcal{P}(Y)&=\mathcal{P}(X)=\mathcal{P}(X_1)\mathcal{P}(X_2)...\mathcal{P}(X_n) \\
			&=\prod_{i=1}^{n}\frac{1}{\sqrt{2\pi \sigma}}e^{-\frac{(X_i-\mu)^2}{2\sigma^2}} \\
			&=(2\pi\sigma^2)^{-\frac{n}{2}}e^{-\frac{1}{2\sigma^2}\sum_{i=1}^{n}(X_i-\mu)^2} \\
			&=(2\pi\sigma^2)^{-\frac{n}{2}}e^{-\frac{1}{2\sigma^2}\sum_{i=1}^{n}(X_i^2-2X_i\mu+\mu^2)} \\
			&=(2\pi\sigma^2)^{-\frac{n}{2}}e^{-\frac{1}{2\sigma^2}(\sum_{i=1}^{n}X_i^2-2n\bar{X}\mu+n\mu^2)} \\
		\end{array} \label{eq:p}
	\end{equation}
	For $Y$, we have $Y^\textsuperscript{T}Y=(AX)\textsuperscript{T}(AX)=X\textsuperscript{T}A\textsuperscript{T}AX=X\textsuperscript{T}X$, and $Y^\textsuperscript{T}Y$ can be calculated by $[Y_1, Y_2, ..., Y_n]*[Y_1, Y_2, ..., Y_n]^\textsuperscript{T} = \sum_{i=1}^{n}Y_i^2$, so $\sum_{i=1}^{n}Y_i^2=\sum_{i=1}^{n}X_i^2$. Besides, $Y_1=\frac{1}{\sqrt{n}}(X_1, X_2,...X_n)=\sqrt{n}\bar{X}$, so $\bar{X}=\frac{1}{\sqrt{n}}Y_1$. Replace $X$ in Eq~\ref{eq:p} with $Y$, we get:
	\begin{equation}
		\begin{array}{cl}
			\mathcal{P}(Y)&=(2\pi\sigma^2)^{-\frac{n}{2}}e^{-\frac{1}{2\sigma^2}(\sum_{i=1}^{n}Y_i^2-2\sqrt{n}Y_1\mu+n\mu^2)}\\
			&=(2\pi\sigma^2)^{-\frac{n}{2}}e^{-\frac{1}{2\sigma^2}(\sum_{i=2}^{n}Y_i^2+Y_1^2-2\sqrt{n}Y_1\mu+n\mu^2)} \\
			&=(2\pi\sigma^2)^{-\frac{n}{2}}e^{-\frac{1}{2\sigma^2}(\sum_{i=2}^{n}Y_i^2+(Y_1-\sqrt{n}\mu)^2)} \\
			&=\frac{1}{\sqrt{2\pi \sigma}}e^{-\frac{(Y_1-\sqrt{n}\mu)^2}{2\sigma^2}} \frac{1}{\sqrt{2\pi \sigma}}e^{-\frac{Y_2^2}{2\sigma^2}} \frac{1}{\sqrt{2\pi \sigma}}e^{-\frac{Y_3^2}{2\sigma^2}}... 
		\end{array} \label{eq:p2}
	\end{equation}
	We can infer that \textbf{$Y$ is independent of each other} as Eq~\ref{eq:p2} proves that the probability density function of $Y$ can be written as the product of the density functions of its variables. Then, for $S^2$, we have:
	\begin{equation}
		\begin{array}{rl}
			(n-1)S^2&=\sum_{i=1}^{n}(X_i-\bar{X})^2 \\
			&=\sum_{i=1}^{n}(X_i^2-2X_i\bar{X}+\bar{X}^2) \\
			&=\sum_{i=1}^{n}X_i^2+\sum_{i=1}^{n}\bar{X}(\bar{X}-2X_i) \\
			&=\sum_{i=1}^{n}X_i^2-n\bar{X}^2 \\
			&=\sum_{i=1}^{n}Y_i^2-Y_1^2 \\
			&=\sum_{i=2}^{n}Y_i^2
		\end{array} \label{eq:p3}
	\end{equation}
	Therefore, $\bar{X}$ is only affected by $Y_1$, while $S^2$ is affected by $Y_2$ to $Y_n$. As $[Y_1, Y_2, ..., Y_n]$ are independent of each other, we can conclude that the mean and variance are independent of each other. 
\end{proof}
Through the above theorem and proof, we can know that the mean and variance vectors are independent of each other in the optimization process of VAE. Therefore, we expect the two parts that do not affect each other to represent the features after disentanglement.

\section{More experiments} \label{app:b}
\textbf{Experiments on other language models}. In addition to BERT and RoBERTa, we also add the performance comparison of TACIT under the condition of DeBERTa~\cite{HeLGC21}\footnote{huggingface.co/microsoft/deberta-base} and OPT-1.3b~\cite{OPT}\footnote{https://huggingface.co/facebook/opt-1.3b} as the basic language model. We adopt the same experimental settings and obtain experimental results as shown in Table~\ref{tab:a1}. The experimental results show that TACIT can bring performance gain to DeBERTa and OPT-1.3b (+1.44\% and +3.06\%). The reason why the basic performance of OPT-1.3b is only 82.2\% is that its large number of parameters brings the risk of overfitting. 
\begin{table}[htbp]
	\renewcommand{\arraystretch}{0.9}
	\setlength\tabcolsep{1.3mm}
	\centering
	\begin{tabular}{ccc|cc|c} \hline \hline
		& \multicolumn{2}{c}{OPT-1.3b} & \multicolumn{2}{c}{DeBERTa} & \multicolumn{1}{c}{RoBERTa} \\ \hline
		& \multicolumn{1}{l}{Vanilla} & \multicolumn{1}{l}{TACIT} & Vanilla & TACIT & \multicolumn{1}{c}{TACIT} \\ \hline
		B→D   & 89.55 & 89.50 & 91.95 & 92.45 & 92.65 \\
		B→E   & 90.60  & 92.10 & 93.40 & 94.05 & 93.81 \\
		B→K   & 93.90 & 94.05 & 93.75 & 95.25 & 95.03 \\
		D→B   & 57.45 & 64.95 & 93.40 & 94.10 & 93.57 \\
		D→E   & 56.05 & 66.35 & 90.75 & 93.75 & 93.16 \\
		D→K   & 66.25 & 72.50 & 93.55 & 94.70 & 94.40 \\
		E→B   & 86.05 & 89.30 & 90.65 & 92.95 & 92.70 \\
		E→D   & 84.05 & 87.15 & 89.65 & 92.36 & 92.06 \\
		E→K   & 93.50 & 94.55 & 94.95 & 95.55 & 95.87 \\
		K→B   & 89.90 & 89.95 & 91.05 & 93.15 & 93.06 \\
		K→D   & 86.05 & 89.15 & 90.55 & 92.10 & 91.97 \\
		K→E   & 93.10 & 93.55 & 94.30 & 94.90 & 94.57 \\ \hline
		Avg.  & 82.20 & \textbf{85.26} & 92.33 & \textbf{93.77} & 93.57 \\
		\hline \hline
	\end{tabular}%
	\caption{Results based on OPT-1.3b and DeBERTa. The boldface indicates the optimal
		results.}
	\label{tab:a1}%
\end{table}%

\textbf{Spam detection and sentiment multi-classification}. We evaluate the reliability of the proposed approach in more areas (spam detection) as well as in more classes (sentiment multi-classification). For spam detection, we use three publicly available datasets, Ling\footnote{www.kaggle.com/datasets/mandygu/lingspam-dataset}, SMS\footnote{www.kaggle.com/datasets/uciml/sms-spam-collection-dataset}, and Emails\footnote{www.kaggle.com/datasets/jackksoncsie/spam-email-dataset}. To prevent sample imbalance, we sample the datas of different classes 1:1. The number of sampling is determined by the category that contains fewer samples. For sentiment multi-classification, we use sst5\footnote{github.com/doslim/Sentiment-Analysis-SST5} and twitter\footnote{www.kaggle.com/datasets/saurabhshahane/twitter-sentiment-dataset}. For sst5, samples with a score of 1 in are regarded as negative, those with a score of 3 are regarded as neutral, and those with a score of 5 are regarded as positive. We then sample sst5 and twitter, with 1000 samples for each category. Finally, for each dataset, we adopt 5-fold cross-validation, and the partition ratio of training set and validation set is 8:2. The experimental results are shown in Table~\ref{tab:a2}. The experimental results confirm the superiority of TACIT in various tasks. 

\begin{table}[htbp]
	\renewcommand{\arraystretch}{0.9}
	\setlength\tabcolsep{0.1mm}
	\centering
	\begin{tabular}{lccc|ccc} \hline \hline
		& \multicolumn{3}{c}{BERT} & \multicolumn{3}{c}{RoBERTa} \\ \hline
		& \multicolumn{1}{l}{Vanilla} & UDAML  & \multicolumn{1}{l}{TACIT} & \multicolumn{1}{l}{Vanilla} & UDAML  & TACIT \\ \hline
		Ling→SMS & 68.81 & 81.77 & 85.81 & \multicolumn{1}{r}{90.69} & 90.87 & 92.12 \\
		Ling→Email & 76.57 & 78.26 & 79.09 & \multicolumn{1}{r}{81.65} & 81.94 & 83.30 \\
		SMS→Ling & 50.31 & 53.80 & 53.62 & \multicolumn{1}{r}{50.73} & 54.26 & 53.73 \\
		SMS→Email & 51.02 & 54.43 & 53.06 & \multicolumn{1}{r}{50.91} & 54.58 & 53.17 \\
		Email→SMS & 77.71 & 79.39 & 81.25 & \multicolumn{1}{r}{80.72} & 81.36 & 83.94 \\
		Email→Ling & 94.59 & 94.88 & 96.04 & \multicolumn{1}{r}{94.04} & 94.35 & 96.60 \\ \hline
		Avg.  & 69.84 & 73.76 & \textbf{74.81} & 74.79 & 76.23 & \textbf{77.14} \\ \hline \hline
		sst5→twitter & 49.53 & 52.69 & 57.63 & \multicolumn{1}{r}{49.33} & 51.36 & 52.73 \\
		twitter→sst5 & 57.83 & 60.10  & 62.43 & \multicolumn{1}{r}{56.17} & 56.92 & 58.10 \\ \hline
		Avg.  & 53.68 & 56.40  & \textbf{60.03} & 52.75 & 54.14 & \textbf{55.42} \\ \hline \hline
	\end{tabular}%
	\caption{Spam and multi-class classification results. }
	\label{tab:a2}%
\end{table}%

\end{document}